\documentclass{article} % For LaTeX2e
\usepackage[preprint]{colm2025_conference}

\usepackage{microtype}
\usepackage{hyperref}
\usepackage{url}
\usepackage{booktabs}
\usepackage{graphicx}
\usepackage{lineno}
\usepackage{amsmath}
\usepackage{amssymb} % Added for \mathbb, \mathcal
\usepackage{amsthm}
\definecolor{darkblue}{rgb}{0, 0, 0.5}
\hypersetup{colorlinks=true, citecolor=darkblue, linkcolor=darkblue, urlcolor=darkblue}
\usepackage{listings}
\usepackage{float} 

\newtheorem{definition}{Definition}

\newtheorem{lemma}{Lemma}
\newtheorem{proposition}{Proposition}

\makeatletter
\renewenvironment{proof}[1][\proofname]{%
  \par\pushQED{\qed}%
  \normalfont                     % body upright
  \topsep6\p@\@plus6\p@\relax
  \trivlist
  \item[\hskip\labelsep\itshape #1.]% italic heading only
}{%
  \popQED\endtrivlist\@endpefalse
}
\makeatother

\title{Top-b: Entropic Regulation of Relative Probability Bands in Autoregressive Language Processes}

\author{Deepon Halder \\
AI4Bharat\\
IIEST Shibpur \\
\texttt{deeponh.2004@gmail.com}
\AND
Raj Dabre \\
IIT Madras \\
AI4Bharat
}

\begin{document}

\ifcolmsubmission
\linenumbers
\fi

\maketitle

\begin{abstract}
Probabilistic language generators are theoretically modeled as discrete stochastic processes, yet standard decoding strategies (Top-$k$, Top-$p$) impose static truncation rules that fail to accommodate the dynamic information density of natural language. This misalignment often forces a suboptimal trade-off: static bounds are either too restrictive for high-entropy creative generation or too permissive for low-entropy logical reasoning. In this work, we formalize the generation process as a trajectory through a relative probability manifold. We introduce \textbf{Top-b} (Adaptive Relative Band Sampling), a decoding strategy that regulates the candidate set via a dynamic bandwidth coefficient coupled strictly to the instantaneous Shannon entropy of the model's distribution. We provide a theoretical framework demonstrating that Top-b acts as a variance-minimizing operator on the tail distribution. Empirical validation on GPQA and GSM8K benchmarks indicates that Top-b significantly reduces generation entropy and inter-decoding variance while maintaining competitive reasoning accuracy, effectively approximating a self-regulating control system for autoregressive generation. \\
\end{abstract}

\section{Introduction}

We formulate the problem of autoregressive text generation as the realization of a discrete stochastic process $Y = \{Y_t\}_{t=1}^{\infty}$. In this framework, a stochastic process is simply a sequence of random variables where the outcome of the current step depends on the history of previous steps. Modern Large Language Models (LLMs) function as estimators of the conditional probability distribution $P(Y_t \mid Y_{<t})$, assigning a probability mass to every possible token in a vocabulary $\mathcal{V}$.

The central challenge in decoding is transforming this continuous probability distribution into a single discrete choice $y_t$. This requires a decision rule, or \textit{decoding strategy}, that balances two competing objectives: \textit{fidelity} (adhering to the model's high-confidence predictions) and \textit{diversity} (allowing for variability in expression).

Standard stochastic strategies, such as Top-$k$ \citep{fan2018hierarchical} and Top-$p$ (Nucleus) \citep{holtzman2020curious}, employ \textbf{static truncation}. They restrict the sampling pool (the set of valid next tokens) using fixed hyperparameters, either a fixed count $k$ or a fixed cumulative probability mass $p$.
\begin{figure}[t] % [t] places it at the top of the page, which looks best for large 2x2 grids
    \centering
    % You can use .pdf for vector crispness if you compile with pdflatex/xelatex
    \includegraphics[width=\linewidth]{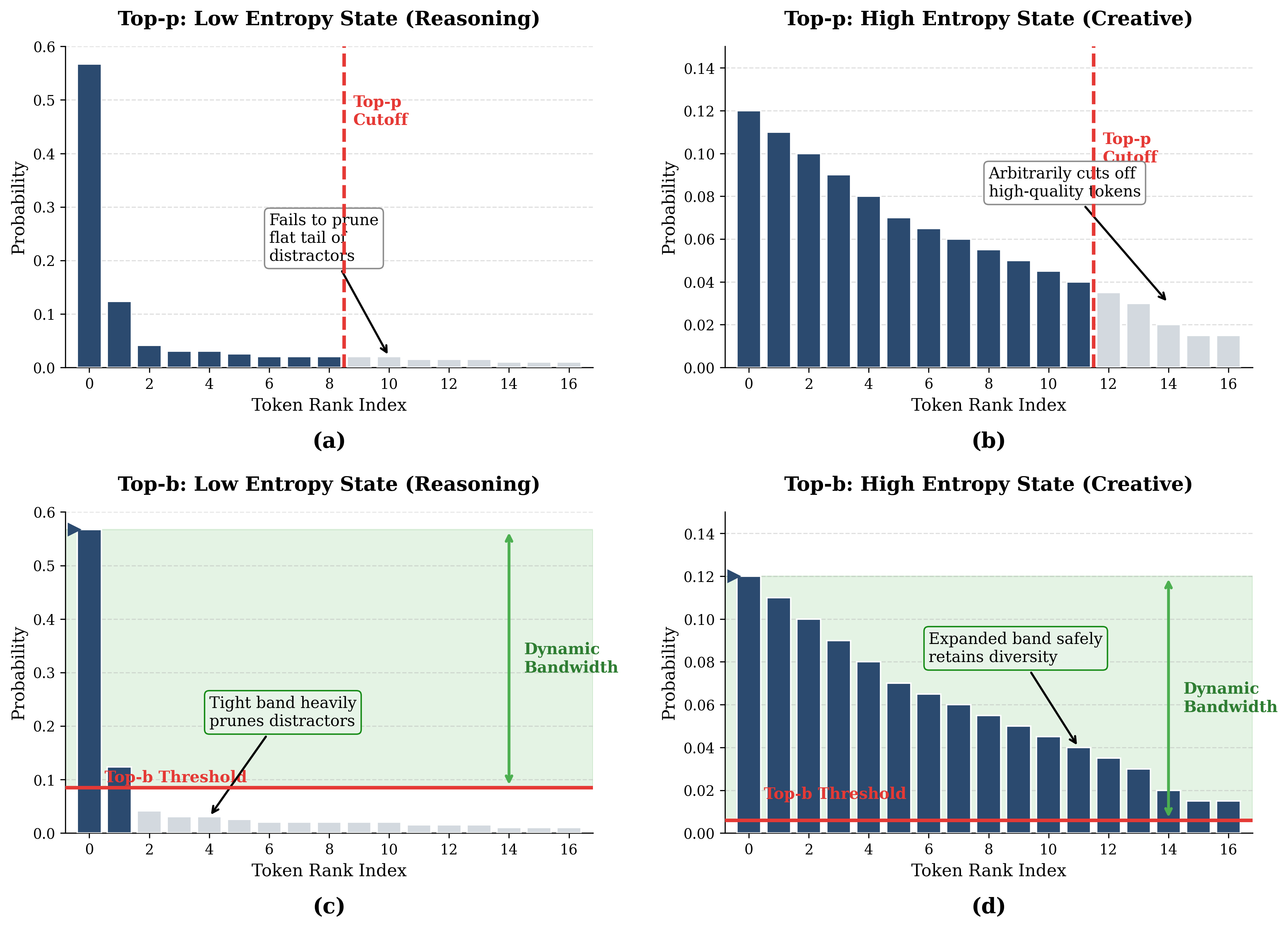} 
    \caption{\textbf{Structural comparison of static cumulative truncation (Top-$p$) versus entropy-regulated relative thresholding (Top-$b$).} \textbf{(a)} In low-entropy reasoning regimes, Top-$p$ admits a long tail of low-probability distractor tokens, increasing the risk of logical incoherence. \textbf{(b)} In high-entropy creative regimes, Top-$p$'s static cumulative threshold arbitrarily truncates viable tokens, artificially restricting diversity. In contrast, Top-$b$ establishes a dynamic probability band anchored to the distribution mode ($p_{\max}$). \textbf{(c)} Under low entropy, the Top-$b$ bandwidth strictly contracts to prune the distractor tail and enforce deterministic reasoning. \textbf{(d)} Under high entropy, the bandwidth expands to safely retain linguistic diversity. This illustrates how Top-$b$ continuously adapts its sampling support to the local information density of the language process.}
    \label{fig:top_b_innovation}
\end{figure}
This static approach is theoretically suboptimal because the \textit{uncertainty} of the language process is non-stationary; it fluctuates wildly over time.
\begin{itemize}
    \item In \textbf{low-entropy states} (e.g., solving an arithmetic equation), the distribution is peaked, meaning most probability mass is concentrated on one or two tokens. Here, a wide static threshold (like $p=0.9$) is too permissive, allowing the inclusion of \textit{tail events}, low-probability outcomes that, in this context, manifest as logical errors or hallucinations.
    \item In \textbf{high-entropy states} (e.g., beginning a creative story), the distribution is flat or uniform. Here, a static threshold might be too restrictive, artificially truncating valid creative options.
\end{itemize}

The primary contributions of this work are as follows:

\begin{itemize}
\item \textbf{Entropy-regulated decoding.} We introduce \textbf{Top-b}, a decoding strategy that defines the sampling support through a relative probability band anchored to the distribution mode and dynamically scaled by the normalized Shannon entropy.

\item \textbf{Variance-minimizing interpretation.} We analyze the structural behavior of Top-b and argue that entropy-scaled relative thresholding acts as a variance-minimizing operator on the tail distribution, stabilizing the generation trajectory in low-entropy reasoning regimes.

\item \textbf{Empirical validation on reasoning benchmarks.} Experiments on GPQA and GSM8K demonstrate that Top-b achieves competitive or improved accuracy while consistently producing lower generation entropy and significantly reduced variance across decoding runs compared with Top-$k$, Top-$p$, Min-$p$, and entropy-aware baselines.
\end{itemize}
\section{Related Work}

The landscape of autoregressive decoding has shifted from fixed heuristics toward adaptive, information-theoretic constraints. We categorize these developments into static truncation, entropy-driven sampling, and relative thresholding.

\subsection{Static and Distribution-Agnostic Truncation}
Standard decoding strategies such as Top-$k$ \citep{fan2018hierarchical} and Top-$p$ (Nucleus) sampling \citep{holtzman2020curious} rely on static hyperparameters to prune the long tail of the distribution. While Top-$k$ limits the support size $|\mathcal{S}|$ to a constant, Top-$p$ allows the support size to vary such that the cumulative mass $\sum_{y \in \mathcal{S}} P(y)$ exceeds a threshold $p$. However, both methods are distribution-agnostic in their "strictness"; a static $p=0.9$ may be too restrictive for high-entropy creative tasks and too permissive for low-entropy logical reasoning, where the remaining $0.1$ mass may consist of thousands of "distractor" tokens in large vocabularies.

\subsection{Entropy-Aware and Information-Theoretic Sampling}
Recent work has sought to couple the truncation process to the model's internal uncertainty. \textbf{Typical Sampling} \citep{meister2023locally} selects tokens whose information content is close to the conditional Shannon entropy, grounded in the property of the Typical Set. \textbf{Mirostat} \citep{basu2021mirostat} utilizes a feedback loop to maintain a constant target cross-entropy, effectively treating decoding as a control system problem.

More recently, \textbf{Eta Sampling} \citep{hewitt2022truncation} and \textbf{Epsilon Sampling} introduce simple probability bounds, but often fail to account for the relative distance between the mode and the tail. \textbf{REAL} (Residual Entropy-Aware Length) \citep{zhao2024real} and \textbf{p-less} sampling \citep{liu2024pless} utilize Rényi entropy and collision entropy to achieve hyperparameter-free truncation. Top-b extends this lineage by utilizing the normalized Shannon entropy to modulate the \textit{relative} width of the probability band, rather than just the absolute probability mass.

\subsection{Relative Thresholding and Min-p}
The most direct predecessor to our work is \textbf{Min-p} sampling \citep{grimstad2024minp}, which introduces a threshold $\tau = \alpha \cdot p_{\max}$. Min-p ensures that a token is only sampled if its likelihood is within a fixed fraction of the most likely token. While Min-p is more robust than Top-p to changes in vocabulary size, its scaling factor $\alpha$ remains a static hyperparameter. 

Top-b generalizes Min-p by formalizing $\alpha$ as a dynamic function of the local state manifold. Specifically, while Min-p assumes a constant tolerance for noise, Top-b posits that the "bandwidth of viability" should expand or contract based on the instantaneous unpredictability ($\eta_t$) of the stochastic process. This bridges the gap between the precision of greedy search ($H \to 0$) and the exploratory nature of random walks ($H \to H_{\max}$).

\subsection{Theoretical Stability in Reasoning}
The pursuit of stability in Chain-of-Thought (CoT) reasoning has led to strategies like \textbf{Self-Consistency} \citep{wang2023self} and \textbf{Best-of-N} sampling. These methods mitigate variance post-generation. In contrast, Top-b acts as an \textit{online} variance-minimizing operator. By pruning the tail distribution proportional to the model's confidence, Top-b minimizes the "branching factor" of errors during autoregressive generation, preventing the accumulation of drift in long-context reasoning tasks.

\section{Formalism of Language Processes}

We analyze the generation process through the lens of information theory.

\subsection{The Discrete Language Process}
Let $\mathcal{V}$ be a finite vocabulary of size $|\mathcal{V}|$. We define the language process as a sequence of random variables $Y_1, Y_2, \dots$, where each $Y_t$ takes a value $y \in \mathcal{V}$. The behavior of the system at time $t$ is fully described by the conditional probability vector $\mathbf{p}_t$:
\begin{equation}
    \mathbf{p}_t(y) = P(Y_t = y \mid Y_{<t} = y_{<t})
\end{equation}
This vector represents the model's estimate of the likelihood of every possible next word, given the history $y_{<t}$.

\subsection{Quantifying Uncertainty and The Mode}
To govern the sampling process, we must quantify the model's confusion at any given step. We utilize \textbf{Shannon Entropy} ($H$), a measure of the unpredictability of the state:
\begin{equation}
    H(\mathbf{p}_t) = - \sum_{y \in \mathcal{V}} \mathbf{p}_t(y) \log \mathbf{p}_t(y)
\end{equation}
A low $H(\mathbf{p}_t)$ implies the model is confident (the distribution is sharp/peaked). A high $H(\mathbf{p}_t)$ implies the model is uncertain (the distribution is flat/uniform).

We also identify the \textbf{Mode} of the distribution, $m_t$, which is simply the single most likely token:
\begin{equation}
    m_t = \arg\max_{y \in \mathcal{V}} \mathbf{p}_t(y), \quad \text{with probability } p_{\max} = \mathbf{p}_t(m_t)
\end{equation}
The mode serves as our anchor point. In our framework, the validity of any other token $y$ is determined by comparing its probability $\mathbf{p}_t(y)$ against the anchor $p_{\max}$.

\section{Locally Adaptive Relative Sampling (Top-b)}

We introduce the concept of an \textbf{Entropy-Regulated Probability Band}. The fundamental hypothesis is that the tolerance for deviating from the mode should be inversely proportional to the model's confidence.

\subsection{Defining the Dynamic Support}

In probability theory, the \textit{support} of a distribution is the set of outcomes with non-zero probability. In decoding, we define a \textit{truncated support} $\mathcal{S}_{b,t}$, which is the subset of the vocabulary we actually consider for sampling.

Let $p_i$ denote the probability assigned to the $i$-th token and let
\[
p_{\max} = \max_i p_i .
\]
We define a dynamic threshold
\[
t = p_{\max} \cdot (1 - \text{bandwidth}_t).
\]
Only tokens satisfying $p_i \ge t$ are considered during sampling.

\begin{definition}[Top-b Support Set]
Let $\text{bandwidth}_t \in (0, 1]$ be the adaptive bandwidth coefficient. The Top-b support set $\mathcal{S}_{b,t}$ consists of all tokens $y$ satisfying:
\begin{equation}
    \mathcal{S}_{b,t} = \{ y \in \mathcal{V} \mid \mathbf{p}_t(y) \ge (1 - \text{bandwidth}_t) \cdot p_{\max} \}.
\end{equation}
\end{definition}

Here, $(1 - \text{bandwidth}_t)$ acts as a relative threshold. If the bandwidth is small, the threshold is close to $p_{\max}$ and only tokens nearly as probable as the best one are retained. If the bandwidth increases, the threshold decreases, allowing more tokens into the support.

\subsection{Adaptive Bandwidth via Entropy}

Let the probability distribution over tokens be $p = \{p_i\}$ and let the Shannon entropy be
\[
H(p) = - \sum_i p_i \log p_i .
\]
Let
\[
H_{\max} = \log(|\mathcal{V}|),
\]
the maximum possible entropy (achieved under a uniform distribution).

We define the adaptive bandwidth as
\begin{equation}
    \text{bandwidth}_t 
    = \text{base\_bandwidth} 
    \times \left( 1 + \frac{H(p)}{H_{\max}} \right),
\end{equation}
where $\text{base\_bandwidth} \in (0,1)$ is a fixed hyperparameter.

This implies:
\begin{enumerate}
    \item When $H \to 0$,
    \[
    \text{bandwidth}_t \to \text{base\_bandwidth}.
    \]
    \item When $H \to H_{\max}$,
    \[
    \text{bandwidth}_t \to 2 \cdot \text{base\_bandwidth}.
    \]
\end{enumerate}

Thus the bandwidth scales linearly with normalized entropy, increasing the admissible support as uncertainty grows.

\subsection{Theoretical Behavior}

This formulation ensures the decoding strategy adapts smoothly to the entropy of the probability landscape:

\begin{proposition}[Entropy-Scaled Constraints]
The Top-b mechanism exhibits the following asymptotic behaviors:
\begin{enumerate}
    \item \textbf{Low-Entropy Regime ($H \to 0$):}  
    The bandwidth approaches $\text{base\_bandwidth}$.  
    The threshold becomes
    \[
    t \approx p_{\max}(1 - \text{base\_bandwidth}),
    \]
    restricting sampling to high-probability tokens near the mode.  
    This behavior approximates greedy decoding when the distribution is sharp.

    \item \textbf{High-Entropy Regime ($H \to H_{\max}$):}  
    The bandwidth increases to $2 \cdot \text{base\_bandwidth}$.  
    The threshold decreases accordingly:
    \[
    t \approx p_{\max}(1 - 2 \cdot \text{base\_bandwidth}).
    \]
    The support expands, encouraging diversity as the distribution becomes flatter.
\end{enumerate}
\end{proposition}

\section{Experiments}

We evaluate the hypothesis that Top-b acts as a stabilizer for the language process on high-reasoning tasks. We utilize \texttt{google/gemma-3-4b-it} as the underlying stochastic process generator.

\subsection{Experimental Settings}
We compare against:
\begin{itemize}
    \item \textbf{Stochastic Baselines:} Top-k ($k=40$), Top-p ($p=0.9$), Min-p ($\alpha=0.05$).
    \item \textbf{Entropic Baselines:} Eta Sampling .
\end{itemize}
Datasets include \textbf{GPQA} (graduate-level reasoning) and \textbf{GSM8K} (arithmetic reasoning).

\section{Results}

\subsection{GPQA Evaluation}

Table \ref{tab:gpqa} details the accuracy and average sequence entropy across decoding strategies. Top-b achieves a competitive accuracy of 0.2285 while maintaining the \textbf{statistically lowest average entropy} (0.1579).

\begin{table}[t]

\caption{GPQA accuracy and average entropy. Top-b minimizes average entropy, indicating a reduction in generation stochasticity and confusion.}

\label{tab:gpqa}

\begin{center}

\begin{tabular}{lcc}

\toprule

\textbf{Method} & \textbf{Accuracy} & \textbf{Avg\_Entropy} \\

\midrule

\textbf{Top-b (Ours)} & 0.2285 & \textbf{0.1579} \\

Eta & 0.2099 & 0.1592 \\

Min-p & 0.2199 & 0.1714 \\

Epsilon & \textbf{0.2389} & 0.1712 \\

Top-p & 0.2132 & 0.1764 \\

Temperature & 0.2121 & 0.1774 \\

Top-k & 0.2333 & 0.1815 \\

\bottomrule

\end{tabular}

\end{center}

\end{table}

\textbf{Implications of Entropy Minimization}
The statistically significant reduction in entropy observed with Top-b is non-trivial. In autoregressive generation, token-level entropy functions as a metric of decisiveness; elevated entropy correlates with diffuse probability mass and an increased likelihood of off-distribution generation. By explicitly minimizing entropy, Top-b compels the model to converge on sharper predictions. Unlike Top-p, which may preserve a broad candidate tail even under marginal uncertainty, Top-b imposes a principled constraint, anchoring generation to high-likelihood regions unless the underlying distribution topology explicitly necessitates expansion. (For concrete, token-level examples demonstrating this aggressive tail-pruning {across various contexts, see Appendix \ref{app:case_studies})}

\textbf{Variance Reduction and Stability.}

Figure \ref{gpqa_acc} illustrates the variance in performance across random seeds. Top-b exhibits the \textbf{lowest variance}, implying higher deterministic stability. This suggests that the method is robust to perturbations in token probabilities, a critical property for reproducible reasoning chains. By dampening fluctuations without sacrificing accuracy, Top-b offers a more reliable alternative to purely stochastic sampling.

\begin{figure}[h]

\centering

\includegraphics[width=0.80\linewidth]{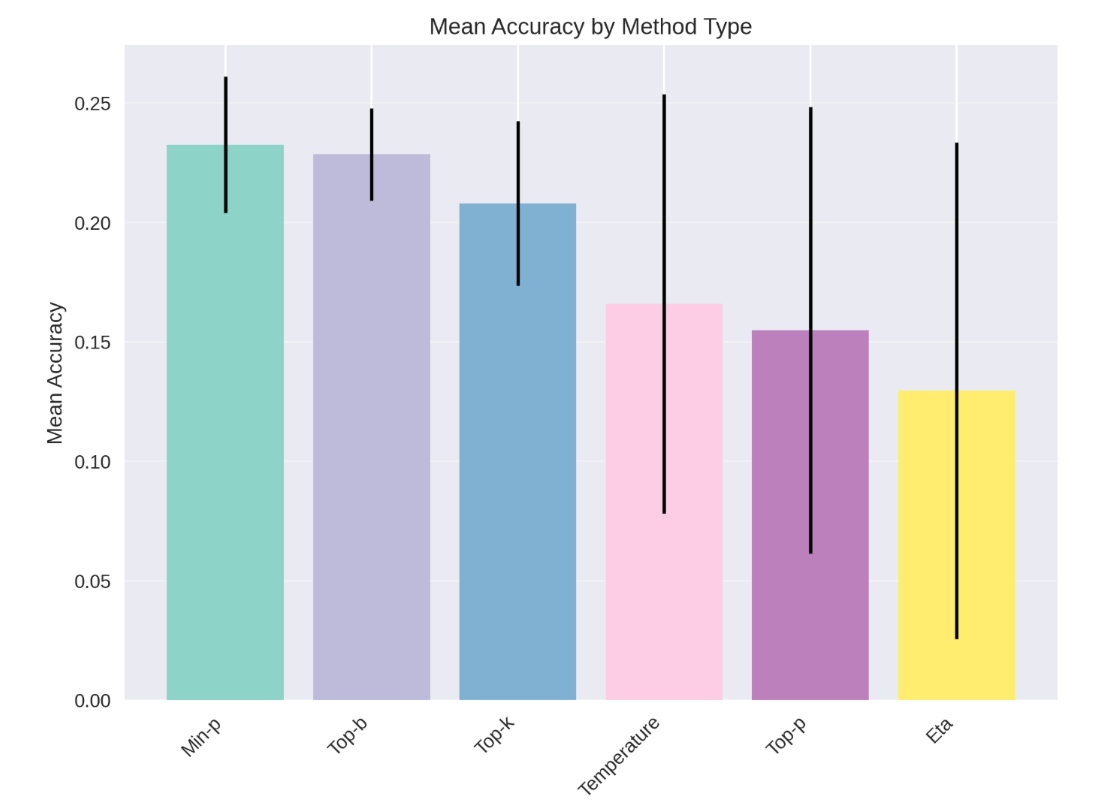}

\caption{Accuracy variance across random seeds on GPQA. Top-b exhibits the lowest variance, indicating higher deterministic stability compared to Top-p and stochastic sampling methods.}

\label{gpqa_acc}

\end{figure}

\textbf{Entropy Dynamics: Top-b vs.\ Top-p.}

We analyzed the entropy trajectory across generated sequences (Figure~\ref{entropyanal}). While initial dynamics are similar, Top-p sustains elevated entropy in the latter sequence segments due to its cumulative probability threshold, which admits wider candidate sets even as confidence should theoretically sharpen. In contrast, Top-b induces a monotonic reduction in entropy, frequently driving it toward zero in the terminal phase of generation. This confirms that Top-b effectively prunes divergent branches, minimizing the accumulation of uncertainty over long-context reasoning.

\begin{figure}[h]

\centering

\includegraphics[width=0.8\linewidth]{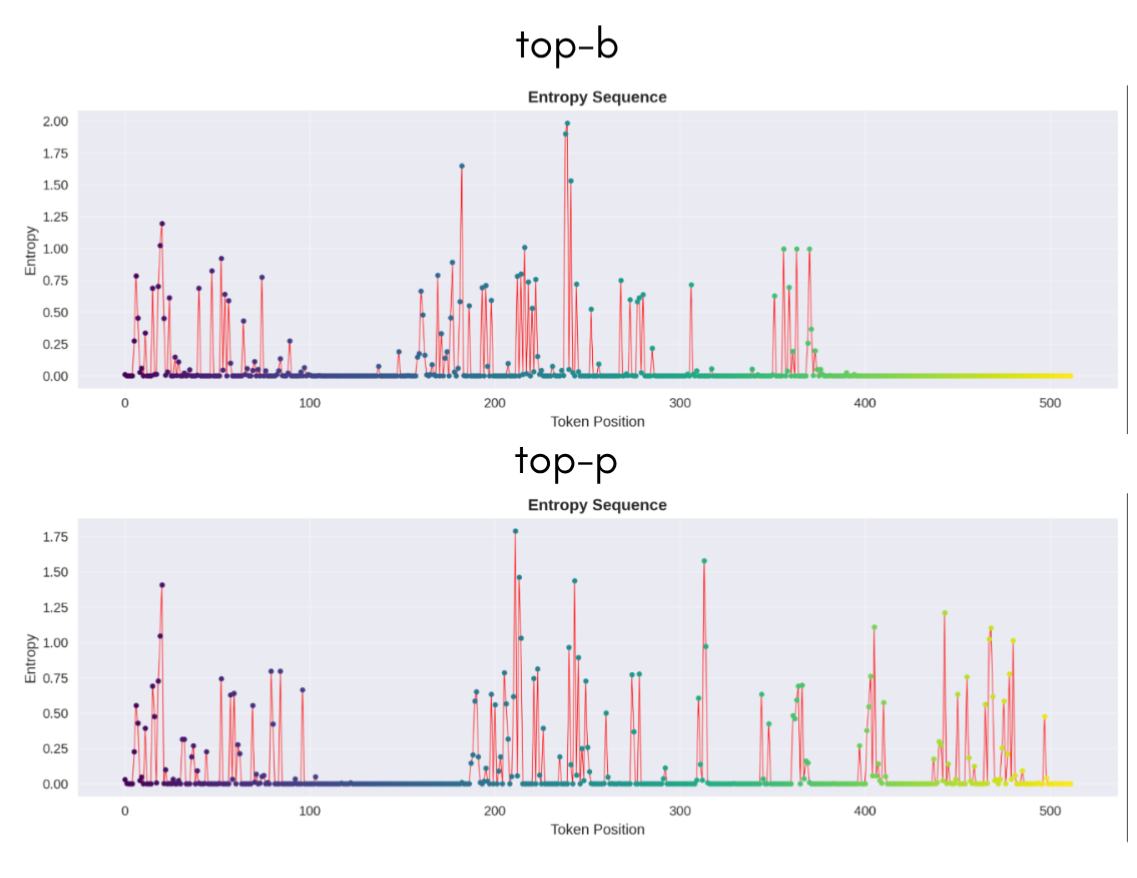}

\caption{Entropy trajectory over generation steps for Top-b and Top-p sampling. Top-b induces a monotonic reduction in entropy, while Top-p maintains higher entropy in later stages due to broader candidate sets.}

\label{entropyanal}

\end{figure}

\subsection*{Interaction with Temperature Scaling}

We investigated the coupling between the bandwidth parameter $b$ and temperature $T$ (Figure~\ref{gpqavar}). The relationship is non-monotonic; mid-range values ($b \approx 0.30$) maximize performance at higher temperatures ($T=2.5$). The results suggest that Top-b acts as a regularizer for high-temperature settings, suppressing the tail noise typically amplified by scaling $T$, provided $b$ is not set so aggressively as to induce mode collapse.

\begin{figure}[h]

\centering

\includegraphics[width=0.80\linewidth]{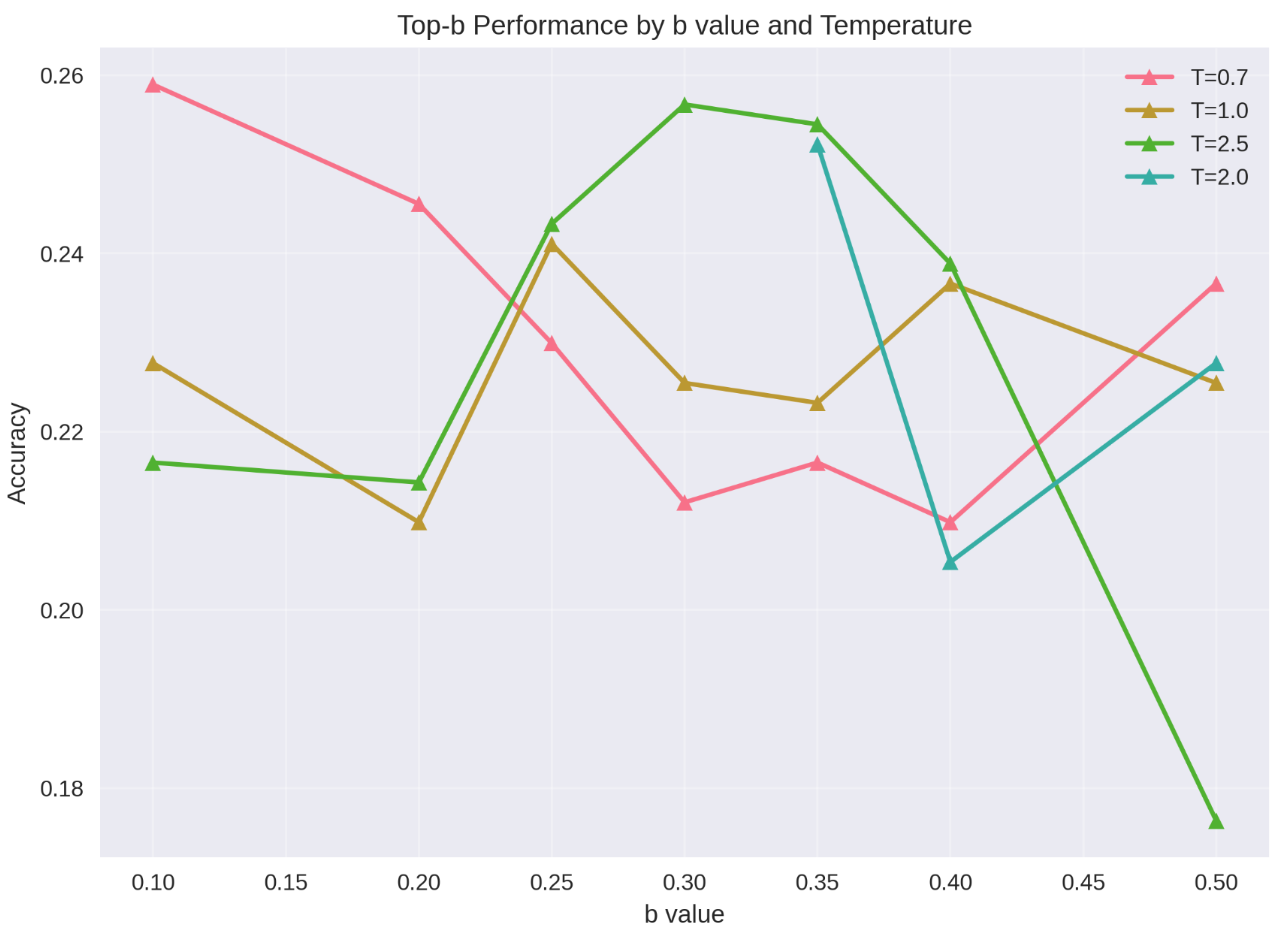}

\caption{Interaction between bandwidth parameter $b$ and temperature $T$ on GPQA performance. Mid-range $b$ values regularize high-temperature sampling, improving accuracy without inducing mode collapse.}

\label{gpqavar}

\end{figure}

\subsection{GSM8K Evaluation}

Preliminary evaluation on GSM8K (1000 samples, CoT split) further validates these findings.

\begin{table}[h]

\centering

\begin{tabular}{lcc}

\hline

\textbf{Method} & \textbf{Average Accuracy} & \textbf{Variance} \\

\hline

\textbf{Top-$b$ (Ours)} & \textbf{0.2680} & $1.136 \times 10^{-4}$ \\

Top-$p$ & 0.2613 & $1.1207 \times 10^{-3}$ \\

Top-$k$ & 0.2093 & $3.797 \times 10^{-4}$ \\

Min-$p$ & 0.2427 & $2.380 \times 10^{-4}$ \\

\hline

\end{tabular}

\caption{Preliminary results on GSM8K demonstrating Top-b's superior accuracy and variance reduction.}

\label{tab:gsm8k_results}

\end{table}

As shown in Table \ref{tab:gsm8k_results}, Top-b achieves the highest accuracy (\textbf{0.2680}) with an order-of-magnitude lower variance ($1.136 \times 10^{-4}$) compared to Top-p. This supports the hypothesis that Top-b effectively filters arithmetic distractors, low-probability tokens that disrupt the logical flow of Chain-of-Thought reasoning.

\section{Discussion and Limitations}

Top-b represents a shift toward \textit{uncertainty-regulated decoding}, where the model's intrinsic confidence governs its generative freedom.

\textbf{The Expressivity-Stability Trade-off.} While our primary focus was logical reasoning, the adaptive formulation of Top-b theoretically extends to creative generation. In creative domains characterized by inherently high entropy (flat distributions), Top-b naturally expands the sampling band, preserving expressivity without requiring manual hyperparameter relaxation.

\textbf{Computational Considerations.} The primary overhead of Top-b is the computation of Shannon entropy at each decoding step. However, as this is a vector operation over $V$ (typically $32k - 128k$), the latency is negligible relative to the $O(N^2)$ complexity of the attention mechanism.

\section{Conclusion}

We have introduced \textbf{Top-b}, an adaptive sampling mechanism that dynamically bounds the search space based on normalized entropy. By coupling the decoding bandwidth to the model's instantaneous uncertainty, Top-b mitigates the risks of static thresholding. Empirical results on GPQA and GSM8K demonstrate that Top-b minimizes sequence entropy and variance while maintaining competitive accuracy. This establishes Top-b not merely as a heuristic, but as a theoretically grounded method for enhancing the reliability of Large Language Model generation.

\section{Acknowledgments}

This research was conducted at AI4Bharat, IIT Madras. We acknowledge the support of the computing resources provided.

\bibliography{colm2025_conference}

\bibliographystyle{colm2025_conference}

\appendix

\section{Future Work}

We intend to extend this evaluation to diverse architectures, including LLaMA \citep{llama3_2024} and Qwen \citep{qwen2_2024}, and conduct scaling laws analysis across model sizes (2B to 405B). Future benchmarks will encompass multilingual tasks and complex mathematical reasoning to further stress-test the entropy-adaptation hypothesis.

\section{Entropy-Induced Branching}

Token-level entropy can be theoretically modeled as a branching factor in the generation tree. High entropy regions represent nodes with multiple viable edges (continuations), whereas low entropy regions approximate linear paths. Top-b effectively prunes these high-entropy branches by enforcing a stricter relative probability band, collapsing the generation tree onto the path of maximum likelihood (Figure \ref{fig:entropy-branching}). This diagnostic perspective suggests that if alternative continuations persist under Top-b, they likely represent deeply memorized or strongly weighted priors within the model's parameter space.

\begin{figure}[h]

\centering

\includegraphics[width=0.60\linewidth]{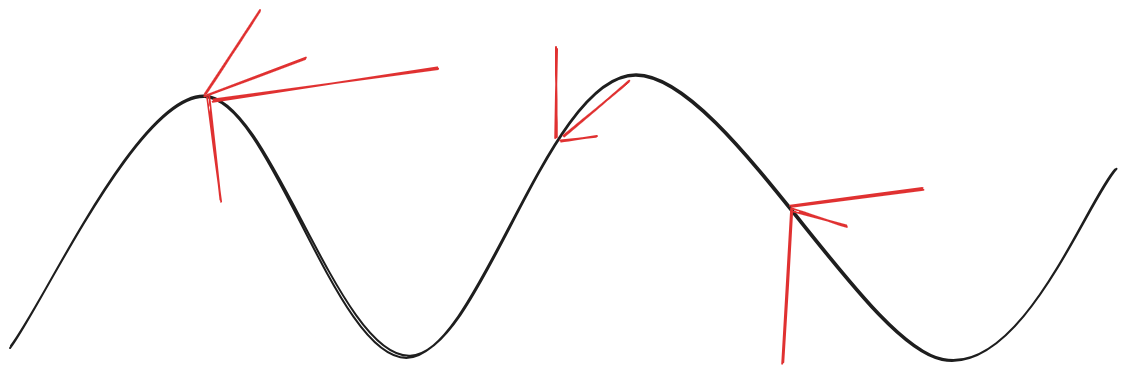}

\caption{Schematic of entropy-induced branching. Top-b acts as a pruning operator, collapsing diffuse branches into a single high-likelihood continuation.}

\label{fig:entropy-branching}

\end{figure}
\section{Proof of Maximum Entropy Bound}
\label{app:proofs}
\begin{lemma}
For a vocabulary $\mathcal{V}$ of size $n$, $H_{\max} = \log n$.
\end{lemma}
\begin{proof}
The Shannon entropy is strictly concave. It is maximized when the distribution is uniform, i.e., $p_i = \frac{1}{n} \forall i$.
\[
H(\mathbf{p}) = - \sum_{i=1}^n \frac{1}{n} \log \frac{1}{n} = - \left( \log \frac{1}{n} \right) \sum_{i=1}^n \frac{1}{n} = \log n.
\]
\end{proof}

\section{Token-Level Pruning Case Studies}
\label{app:case_studies}

To empirically demonstrate the theoretical behavior outlined in Section 4, we examine the next-token probability distributions and the resulting candidate sets across different decoding strategies. Table \ref{tab:case_studies} presents the original token probabilities and the renormalized probabilities for Top-$b$, Top-$p$, Top-$k$, and Min-$p$ across various semantic contexts.

The data reveals a consistent pattern: standard cumulative thresholding (Top-$p = 0.9$) routinely over-admits tokens in low-entropy regimes, whereas Top-$b$ strictly collapses the distribution to the most logical continuations.

\begin{itemize}
    \item \textbf{Arithmetic Reasoning (Low Entropy):} For the prompt \textit{"2+2="}, despite the model assigning 39.75\% probability to the correct token ("4"), Top-$p$ admits 14 total tokens into the candidate set, including "5" and "1", introducing unnecessary stochastic risk. Top-$b$ contracts its bandwidth and selects exactly 1 token ("4").
    \item \textbf{Syntactic Continuation (Low Entropy):} For the prompt \textit{"You will pay for what you have done, she hissed... The battle that ensued"}, the highly probable next token is "was" (46.00\%). Top-$p$ admits an excessive 44 tokens into the sampling pool. Top-$b$ again aggressively prunes the tail, selecting only 1 token.
    \item \textbf{Open-Ended Instruction (Moderate Entropy):} For prompts requiring structural formatting (e.g., \textit{"Describe a time..."}), the probability is distributed among valid structural tokens like newlines (\verb|\n|) and first-person pronouns ("I", "You"). Here, Top-$b$ correctly expands to include 2 to 3 valid tokens, whereas Top-$p$ includes 25 to 43 tokens, many of which are irrelevant distractors.
\end{itemize}
\begin{table}[H]
\centering
\caption{Comparison of selected support sizes ($|\mathcal{S}|$) and renormalized probabilities across different contexts. Top-$b$ consistently isolates the high-probability manifold, severely restricting the distractor tail compared to Top-$p$.}
\label{tab:case_studies}
\resizebox{\textwidth}{!}{%
\begin{tabular}{l l r r r r r}
\toprule
\textbf{Prompt Context} & \textbf{Top Tokens} & \textbf{Orig. Prob (\%)} & \textbf{Top-$b$} & \textbf{Top-$p$} & \textbf{Top-$k$} & \textbf{Min-$p$} \\
\midrule
\multicolumn{7}{l}{\textit{Prompt: "2+2="}} \\
\textbf{Support Size ($|\mathcal{S}|$)} & & & \textbf{1 token} & 14 tokens & 5 tokens & 9 tokens \\
& \texttt{\ 4} & 39.75 & \textbf{100.0} & 44.2 & 54.8 & 46.6 \\
& \texttt{\ 5} & 15.50 & - & 17.2 & 21.5 & 18.3 \\
& \texttt{\ \textbackslash n} & 7.81 & - & 8.7 & 10.8 & 9.2 \\
\midrule
\multicolumn{7}{l}{\textit{Prompt: "A rainbow is an optically brilliant meteorological event resulting from refraction, reflection, and dispersion of"}} \\
\textbf{Support Size ($|\mathcal{S}|$)} & & & \textbf{1 token} & 2 tokens & 5 tokens & 3 tokens \\
& \texttt{\ light} & 59.50 & \textbf{100.0} & 73.0 & 61.8 & 62.8 \\
& \texttt{\ sunlight} & 21.88 & - & 27.0 & 22.8 & 23.1 \\
& \texttt{\ the} & 13.25 & - & - & 14.0 & 14.0 \\
\midrule
\multicolumn{7}{l}{\textit{Prompt: "You will pay for what you have done, she hissed, her blade flashing in the moonlight. The battle that ensued"}} \\
\textbf{Support Size ($|\mathcal{S}|$)} & & & \textbf{1 token} & 44 tokens & 5 tokens & 4 tokens \\
& \texttt{\ was} & 46.00 & \textbf{100.0} & 51.2 & 72.0 & 74.8 \\
& \texttt{\ left} & 6.25 & - & 6.9 & 9.8 & 10.1 \\
& \texttt{\ between} & 5.53 & - & 6.2 & 8.6 & 9.0 \\
\midrule
\multicolumn{7}{l}{\textit{Prompt: "If you could help me write an email to my friends inviting them to dinner on Friday, it would be greatly appreciated."}} \\
\textbf{Support Size ($|\mathcal{S}|$)} & & & \textbf{3 tokens} & 43 tokens & 5 tokens & 14 tokens \\
& \texttt{\ \textbackslash n} & 23.88 & \textbf{52.9} & 26.5 & 41.5 & 31.0 \\
& \texttt{\ I} & 11.25 & \textbf{25.0} & 12.6 & 19.6 & 14.7 \\
& \texttt{\ \textbackslash n\textbackslash n} & 9.94 & \textbf{22.1} & 11.1 & 17.4 & 12.9 \\
\midrule
\multicolumn{7}{l}{\textit{Prompt: "Describe a time when you had to make a difficult decision."}} \\
\textbf{Support Size ($|\mathcal{S}|$)} & & & \textbf{2 tokens} & 25 tokens & 5 tokens & 6 tokens \\
& \texttt{\ \textbackslash n} & 34.50 & \textbf{56.2} & 38.5 & 45.2 & 44.2 \\
& \texttt{\ You} & 27.00 & \textbf{43.8} & 30.1 & 35.2 & 34.5 \\
& \texttt{\ \textbackslash n\textbackslash n} & 8.75 & - & 9.7 & 11.5 & 11.2 \\
\bottomrule
\end{tabular}%
}
\end{table}

\end{document}